\documentclass{article}
\usepackage{arxiv}
\usepackage{times}
\usepackage[round]{natbib}
\renewcommand{\headeright}{Preprint}
\renewcommand{\undertitle}{Preprint}
\renewcommand{\shorttitle}{Why Adaptive Optimizers Underestimate Rare Tokens}
\usepackage{hyperref}
\usepackage{url}
\usepackage{amsmath,amssymb,amsthm}
\usepackage{graphicx}
\usepackage{booktabs}
\usepackage{xcolor}

\newtheorem{theorem}{Theorem}
\newtheorem{proposition}[theorem]{Proposition}
\theoremstyle{remark}

\newcommand{\one}{\mathbf{1}}
\newcommand{\R}{\mathbb{R}}
\newcommand{\E}{\mathbb{E}}
\newcommand{\wbar}{\bar{\mathbf{w}}}
\newcommand{\Cov}{\operatorname{Cov}}
\newcommand{\sign}{\operatorname{sign}}

\newcommand{\figorbox}[2]{\IfFileExists{#1}{\includegraphics[width=#2]{#1}}{\fbox{\parbox{#2}{\centering\vspace{2em}\textcolor{red}{Figure not found: #1}\vspace{2em}}}}}
\newcommand{\taborbox}[1]{\IfFileExists{#1}{\input{#1}}{\fbox{\textcolor{red}{Table not found: #1}}}}

\IfFileExists{results/macros.tex}{\newcommand{\ThmCheckMaxRelErr}{$2.1\times 10^{-14}$}
\newcommand{\ThmLimitMaxRelErr}{$2.5\times 10^{-3}$}
\newcommand{\ConsMaxConserving}{$1.6\times 10^{-15}$}
\newcommand{\ConsMinBreaking}{$3.6\times 10^{-1}$}
\newcommand{\STPerMaxDev}{$0.009$}
\newcommand{\STRandDevLow}{$-0.078$}
\newcommand{\STRandDevHigh}{$-5.88$}
\newcommand{\STAdamDevMax}{$0.06$}
\newcommand{\UniRMSKfive}{$-3.55$}
\newcommand{\UniAMSAbs}{$0.047$}
\newcommand{\UniCoupledAbs}{$0.046$}
\newcommand{\UniSGDAbs}{$0.039$}
\newcommand{\UniSignFracBelow}{$99\%$}
\newcommand{\UniSignFracAbove}{$0\%$}
\newcommand{\UniRMSMaxDev}{$3.86$}
\newcommand{\UniAdamMaxDev}{$3.94$}
\newcommand{\UniThKfive}{$-1.50$}
\newcommand{\LMRareSGD}{$-0.22$}
\newcommand{\LMDriftSGD}{$0.00$}
\newcommand{\LMRareAdamA}{$-2.72$}
\newcommand{\LMDriftAdamA}{$-46.83$}
\newcommand{\LMRareAdamB}{$-0.06$}
\newcommand{\LMDriftAdamB}{$-8.45$}
\newcommand{\LMRareRMS}{$-2.65$}
\newcommand{\LMDriftRMS}{$-53.42$}
\newcommand{\LMRareAMS}{$-0.14$}
\newcommand{\LMDriftAMS}{$-5.25$}
\newcommand{\LMRareCoupled}{$-0.18$}
\newcommand{\LMDriftCoupled}{$-0.00$}
\newcommand{\LMRareAdamW}{$-2.25$}
\newcommand{\LMDriftAdamW}{$-16.04$}
\newcommand{\LMThA}{$-1.60$}
\newcommand{\LMThB}{$-0.03$}
}{}
\providecommand{\ThmCheckMaxRelErr}{\PH}\providecommand{\ThmLimitMaxRelErr}{\PH}
\providecommand{\ConsMaxConserving}{\PH}\providecommand{\ConsMinBreaking}{\PH}
\providecommand{\STPerMaxDev}{\PH}\providecommand{\STRandDevLow}{\PH}
\providecommand{\STRandDevHigh}{\PH}\providecommand{\STAdamDevMax}{\PH}
\providecommand{\UniRMSMaxDev}{\PH}\providecommand{\UniAdamMaxDev}{\PH}
\providecommand{\UniRMSKfive}{\PH}\providecommand{\UniThKfive}{\PH}
\providecommand{\UniSGDAbs}{\PH}\providecommand{\UniAMSAbs}{\PH}\providecommand{\UniCoupledAbs}{\PH}
\providecommand{\UniSignFracBelow}{\PH}\providecommand{\UniSignFracAbove}{\PH}
\providecommand{\LMRareSGD}{\PH}\providecommand{\LMRareAdamA}{\PH}\providecommand{\LMRareAdamB}{\PH}
\providecommand{\LMRareRMS}{\PH}\providecommand{\LMRareAMS}{\PH}\providecommand{\LMRareCoupled}{\PH}
\providecommand{\LMRareAdamW}{\PH}\providecommand{\LMThA}{\PH}\providecommand{\LMThB}{\PH}
\providecommand{\LMDriftSGD}{\PH}\providecommand{\LMDriftAdamA}{\PH}\providecommand{\LMDriftAdamB}{\PH}
\providecommand{\LMDriftRMS}{\PH}\providecommand{\LMDriftAMS}{\PH}
\providecommand{\LMDriftCoupled}{\PH}\providecommand{\LMDriftAdamW}{\PH}

\title{Why Adaptive Optimizers Underestimate Rare Tokens:\\Biased Fixed Points in the Softmax Output Layer}
\author{Sangsidhya Kar \\
Department of Statistics\\
Presidency University, Kolkata}

\begin{document}
\maketitle

\begin{abstract}
In the softmax output layer, a rare token receives a small positive logit gradient on most steps and a much larger negative gradient on the few steps when it is the target. SGD simply adds these contributions. Coordinate-wise adaptive methods such as Adam, RMSProp, and sign descent instead divide each update by a running estimate of its magnitude, and that estimate is largest immediately after the token appears. This imbalance has two effects. At the level of the whole output layer, we characterize which optimizers preserve the mean output embedding: every method whose update is linear in past gradients does, as do Kronecker-factored and orthogonalized methods such as Shampoo and Muon. Adam, Adafactor, Lion, and sign descent do not, and for these methods we obtain an exact step-by-step expression for the change. At the level of an individual rare token, the same normalization shifts the training fixed point. In the unigram model, sign descent lowers the logit of every token that occurs in fewer than half of the minibatches at a constant expected rate. For RMSProp with periodic arrivals, we can solve the fixed point in closed form: if a token is absent for at least two consecutive minibatches, its equilibrium probability is strictly below its data frequency for every learning rate, and the ratio tends to $\kappa/(2(e^{\kappa/2}-1))$. Here $\kappa$ is the mean number of steps between occurrences divided by the second-moment time constant $1/(1-\beta_2)$. In the same model, SGD and AMSGrad retain the unbiased fixed point. We test these predictions both in a unigram model and in a small language model trained from a known generating distribution. With random arrivals, the bias is larger than the periodic formula predicts; in the language model, the optimizers with the biased fixed point also fit the generating distribution less well.
\end{abstract}

\section{Introduction}

Language models are trained predominantly with coordinate-wise adaptive optimizers such as Adam \citep{kingma2015adam}, while token frequencies are strongly heavy-tailed \citep{piantadosi2014zipf}. The softmax output layer brings these two features together. For a token $i$, the cross-entropy gradient with respect to its logit is $p_i-y_i$. When $i$ is not the target, this quantity is small and positive; when the rare event $i$ is the target occurs, it is large and negative. SGD accumulates these contributions linearly, so in expectation the token's logit stops drifting when its average predicted probability matches its frequency in the data. A coordinate-wise adaptive optimizer handles the same gradients differently: it divides each coordinate by a running root-mean-square of that coordinate's past gradients. For a rare token, that second-moment estimate is largest immediately after an observation and then decays throughout the long gap before the next one. Thus, the upward update following an observation is strongly suppressed, while the smaller downward updates later in the gap are divided by progressively smaller quantities.

There are two related effects to keep apart. At the vocabulary level, adding the same constant to every logit leaves the softmax probabilities unchanged. The corresponding output-layer gradients sum to zero across the vocabulary, and any optimizer whose update is a linear combination of past gradients preserves the mean output embedding. This is an instance of the conservation laws associated with symmetries in gradient-based learning \citep{kunin2021neural}. \citet{stollenwerk2025coupled} showed that Adam's second-moment normalization breaks this conservation law, producing the common embedding shift observed by \citet{gao2019representation,bis2021too}; that shift has subsequently been connected to output-logit divergence \citep{wortsman2024small,stollenwerk2026oec}. For an untied output layer without weight decay, however, the common shift leaves predictions unchanged.

The more important effect is local to individual rare tokens, because it changes their predicted probabilities. Coordinate-wise normalization shifts the fixed point of training, and the resulting bias remains even as the learning rate tends to zero. The main results are:
\begin{itemize}
\item \textbf{Which optimizers conserve the mean output embedding} (Section~\ref{sec:mean}). Conservation holds for SGD and momentum methods, and also for Kronecker-factored methods such as Shampoo \citep{gupta2018shampoo} and orthogonalized methods such as Muon \citep{jordan2024muon}. For any coordinate-wise method, the change of the mean at each step equals minus the learning rate times the covariance, across the vocabulary, between the momentum and the coordinate-wise scaling.
\item \textbf{Sign descent} (Section~\ref{sec:sign}). In a unigram model, sign descent lowers the logit of every token that appears in fewer than half of the minibatches at a constant expected rate, however small its probability already is.
\item \textbf{The fixed point of RMSProp} (Section~\ref{sec:rmsprop}). For a token that appears once every $N$ steps, we compute the RMSProp fixed point in closed form. For every $N\ge3$ and $\beta_2\in(0,1)$ the token's probability at the fixed point is strictly below its frequency; the ratio depends only on $N$ and $\beta_2$, and tends to $\rho(\kappa)=\kappa/(2(e^{\kappa/2}-1))$ when $N(1-\beta_2)\to\kappa$. In the same model, SGD and AMSGrad \citep{reddi2018convergence} have unbiased fixed points.
\item \textbf{Predictions and experiments} (Section~\ref{sec:exp}). The analysis predicts that, for a token of frequency $q_i\ll1/B$, the bias depends on $\beta_2$ and on the batch size $B$ only through $\kappa=(1-\beta_2)/(q_iB)$. We test this in a unigram model and in a small language model trained on data from a known generating distribution; in the unigram model it holds for $\kappa\lesssim1$ but not for larger $\kappa$, where the bias is also larger than predicted. All experiments run on a CPU with the code provided as supplementary material.
\end{itemize}
The parameter $\kappa$ is not necessarily small in common settings. With $\beta_2=0.95$, as in \citet{wortsman2024small}, and a batch size of $B=4096$, every token with frequency below $1.2\times10^{-5}$ has $\kappa>1$, while frequencies below $1.2\times10^{-6}$ give $\kappa>10$; at the latter value, $\rho(\kappa)\approx0.03$.

\section{Setting}\label{sec:setting}

The output layer maps a representation $h=h(x;\theta)\in\R^d$, produced by the rest of the network with parameters $\theta$, to logits $z=Wh+b$, where $W\in\R^{V\times d}$ and $b\in\R^V$. For a target distribution $t$—whether a one-hot label, a smoothed label, or a teacher distribution—the loss is $\ell(z,t)=\log\sum_je^{z_j}-t^\top z$, with $p=\mathrm{softmax}(z)$. Let $\one$ denote the all-ones vector in $\R^V$, let $\wbar=W^\top\one/V$ be the mean output embedding, and let $\bar b=\one^\top b/V$ be the mean bias. Then the mean logit is $\bar z(x)=\wbar^\top h(x)+\bar b$. Two facts will be used repeatedly.

\textbf{(F1)} We have $\nabla_z\ell=p-t$ and $\one^\top(p-t)=0$. It follows that, for every minibatch, the gradients $G=\frac1{|\mathcal B|}\sum_n(p_n-t_n)h_n^\top$ and $g=\frac1{|\mathcal B|}\sum_n(p_n-t_n)$ obey $\one^\top G=0$ and $\one^\top g=0$.

\textbf{(F2)} When $W$ is not tied to the input embedding, adding $\one u^\top$ to $W$ and $c\one$ to $b$ leaves the loss unchanged for every $u\in\R^d$ and $c\in\R$. The gradients are unchanged as well: $p$ is unchanged, and $\nabla_h\ell=W^\top(p-t)$ is unchanged because $\one^\top(p-t)=0$.

\textbf{Optimizers.} The coordinate-wise methods we consider update each parameter coordinate with
\begin{equation}\label{eq:adam}
m_t=\beta_1m_{t-1}+(1-\beta_1)g_t,\quad v_t=\beta_2v_{t-1}+(1-\beta_2)g_t^2,\quad \theta_{t+1}=\theta_t-\eta_t\frac{m_t/(1-\beta_1^t)}{\sqrt{v_t/(1-\beta_2^t)}+\epsilon}.
\end{equation}
Adam \citep{kingma2015adam} uses $\beta_1>0$; RMSProp \citep{tieleman2012rmsprop} is the case $\beta_1=0$; AMSGrad \citep{reddi2018convergence} replaces the second-moment estimate by its running maximum; Coupled Adam \citep{stollenwerk2025coupled} replaces it, for embedding matrices, by its average over the vocabulary index. Sign descent \citep{bernstein2018signsgd} uses $\theta_{t+1}=\theta_t-\eta\,\sign(g_t)$, and Lion \citep{chen2023lion} applies the sign to an interpolation of momentum and gradient. Adafactor \citep{shazeer2018adafactor} uses a factored second-moment estimate.

\section{The mean output embedding}\label{sec:mean}

\begin{proposition}[Conservation]\label{prop:cons}
Let $G_1,G_2,\dots$ be the gradients of $W$ from any sequence of minibatches and targets, and let the update at step $t$ be $W_{t+1}=W_t-\eta_tU_t$, with $U_t$ of one of the forms:
\textbf{(a)} $U_t=\sum_{s\le t}\alpha_{t,s}G_s$ with scalars $\alpha_{t,s}$ (SGD, heavy-ball and Nesterov momentum, clipping by the global norm);
\textbf{(b)} $U_t=A_tM_tR_t$, with $M_t$ of form (a) and $A_t^\top\one\in\mathrm{span}(\one)$; this includes Shampoo, $A_t=(\epsilon I+\sum_{s\le t}\gamma_{t,s}G_sG_s^\top)^{-1/4}$ and $R_t=(\epsilon I+\sum_{s\le t}\gamma_{t,s}G_s^\top G_s)^{-1/4}$ with $\gamma_{t,s}\ge0$;
\textbf{(c)} $U_t=M_t\,\psi_t(M_t^\top M_t)$, with $M_t$ of form (a) and $\psi_t$ any matrix function; this includes the polar factor of the momentum and its Newton--Schulz approximation used by Muon;
\textbf{(d)} $U_{t,ij}=d_{t,j}M_{t,ij}$, with $M_t$ of form (a) and a scaling that does not depend on the vocabulary index $i$ (Coupled Adam).
Then $\one^\top U_t=0$ and $\wbar_t=\wbar_0$ for all $t$. With decoupled weight decay \citep{loshchilov2019decoupled}, $W_{t+1}=(1-\eta_t\lambda)W_t-\eta_tU_t$, the mean decays geometrically: $\wbar_{t+1}=(1-\eta_t\lambda)\wbar_t$. Forms (a), (b) and (d) apply to $b$ in the same way.
\end{proposition}

The proof in Appendix~\ref{app:proofs} relies only on (F1). For part (b), $G_s^\top\one=0$ gives $L_t\one=\epsilon\one$, so $\one$ is an eigenvector of $L_t$ and of every power of $L_t$. Muon is usually used for hidden layers while the output layer is trained with AdamW \citep{jordan2024muon}; Proposition~\ref{prop:cons} shows that replacing the output-layer update by an orthogonalized one would preserve $\wbar$.

\begin{proposition}[Exact change under coordinate-wise methods]\label{prop:cov}
Let $U_{t,ij}=d_{t,ij}M_{t,ij}$ with $M_t$ of form (a) and arbitrary $d_{t,ij}$; for Adam, $M_t=m_t$ and $d_t=1/((1-\beta_1^t)(\sqrt{v_t/(1-\beta_2^t)}+\epsilon))$. Then, for every column $j$,
\[\wbar_{t+1,j}-\wbar_{t,j}=-\eta_t\,\Cov_i\big(d_{t,ij},M_{t,ij}\big),\qquad \Cov_i(a,c)=\tfrac1V\textstyle\sum_i(a_i-\bar a)(c_i-\bar c).\]
For sign-based updates $U_t=\sign(S_t)$ (sign descent, Lion), $\wbar_{t+1,j}-\wbar_{t,j}=-\frac{\eta_t}{V}(n^+_{t,j}-n^-_{t,j})$, where $n^\pm_{t,j}$ count the positive and negative entries of column $j$ of $S_t$.
\end{proposition}

For the bias ($h\equiv1$), a token that has not appeared recently typically has positive momentum, while its small second moment produces a large scaling factor. Their covariance is therefore positive, so the mean bias decreases. \citet{stollenwerk2025coupled} described the same mechanism in terms of the expected second moments of frequent and rare tokens. Proposition~\ref{prop:cov} gives the per-step identity for every coordinate-wise method considered here; Table~\ref{tab:cons} verifies both propositions numerically.

\textbf{Effect of the common shift.} By (F2) and induction on $t$, if the output layer is untied and has no weight decay, subtracting the vocabulary mean from each update—or applying the same $\mu$-centering directly to $W$ as in \citet{stollenwerk2026oec}—changes $W_t$ only by a term $\one c_t^\top$. In exact arithmetic, every prediction and every other parameter therefore remains unchanged. The shift can still matter through finite-precision effects \citep{liu2026grokking}, weight decay, tied embeddings \citep{press2017using}, and auxiliary objectives such as z-loss \citep{chowdhery2023palm}. It does not matter through the softmax loss itself. The rare-token effect developed next is separate: it changes the predictions, and centering does not remove it.

\section{Biased fixed points for rare tokens}\label{sec:fixed}

\subsection{A unigram model}
We isolate a single token by retaining only the output bias. The parameters are $b\in\R^V$, with prediction $p(b)=\mathrm{softmax}(b)$. At each step, a minibatch of $B$ tokens is sampled independently from $q$; if $c_t\in\mathbb N^V$ contains the token counts, then the stochastic gradient of the average cross-entropy is $g_t=p(b_t)-c_t/B$. Because $\E[c_t/B]=q$, a constant-step method that combines gradients linearly has, in expectation, the usual fixed point $p=q$, which is also the population-loss minimizer. For token $i$, let $\pi_i=1-(1-q_i)^B$ denote the probability of appearing in a minibatch. When $q_iB\ll1$, we have $\pi_i\approx q_iB$, so the mean gap between appearances is $1/\pi_i$. In the language-model setting, $b$ is the output bias and $p(b)$ is replaced by the prediction averaged over the minibatch. Stationarity with respect to $b$ then requires that the context-averaged prediction equal the empirical token distribution; this is the criterion used in Section~\ref{sec:lm}.

\subsection{Sign descent}\label{sec:sign}
\begin{theorem}[Sign descent]\label{thm:sign}
In the unigram model, let $b_{t+1}=b_t-\eta\,\sign(g_t)$ with $\sign(0)=0$, and let $\mathcal F_t$ be the history up to step $t$. For every token $i$,
$\E[b_{t+1,i}-b_{t,i}\mid\mathcal F_t]\le-\eta(1-2\pi_i)$, and almost surely $\limsup_{t\to\infty}(b_{t,i}-b_{0,i})/t\le-\eta(1-2\pi_i)$. If $q_i=0$, then $b_{t,i}=b_{0,i}-\eta t$ exactly. Moreover $\E[\bar b_{t+1}-\bar b_t\mid\mathcal F_t]\le-\eta\,(1-2\min(B,V)/V)$, which is negative when $B<V/2$.
\end{theorem}
The key observation is: If $c_{t,i}=0$, then $g_{t,i}=p_i>0$ and the update is $-\eta$; otherwise the update is at most $\eta$ in the opposite direction (Appendix~\ref{app:proofs}). The theorem concerns logits rather than probabilities, and part of the decrease is a common shift (Section~\ref{sec:mean}). The important consequence is that the downward force on a token appearing in fewer than half of the minibatches does not weaken as its probability becomes small, whereas under SGD that force is $\eta p_i$. For a token that never occurs, including an unused vocabulary entry \citep{land2024fishing}, the decrease is exactly linear.

\subsection{RMSProp: the fixed point for periodic arrivals}\label{sec:rmsprop}
\textbf{Setting.} Fix a token $i$ with $q_iB<1$. Assume (A1) that it appears exactly once every $N\ge2$ steps, namely at $t\equiv0\pmod N$, so $q_i=1/(NB)$; and (A2) that over one period $\log\sum_je^{b_j}$ is constant while $p_i$ changes negligibly, corresponding to small $\eta$. The gradient of $b_i$ is then $g_t=p-a$ at an arrival and $g_t=p$ at every other step, where $a=1/B$ and $p=p_i$. Over one period, the change in $b_i$ is therefore the change in $\log p_i$. Set $x:=a/p$; the unbiased point $p=q_i$ corresponds to $x=N$. With $\epsilon=0$, RMSProp follows $b_{t+1,i}=b_{t,i}-\eta g_t/\sqrt{v_t}$ and $v_t=\beta_2v_{t-1}+(1-\beta_2)g_t^2$.

\begin{theorem}[Fixed point of RMSProp]\label{thm:rms}
Under (A1)--(A2), for all $\beta_2\in(0,1)$ and $N\ge2$:
\begin{enumerate}
\item[(i)] $v_t$ converges to a unique $N$-periodic sequence, along which the change of $\log p_i$ over one period is
\[\Delta(x)=-\eta\Big(\sum_{n=1}^{N-1}\bar v_n^{-1/2}-(x-1)\,\bar v_0^{-1/2}\Big),\quad \bar v_0=\frac{(1-\beta_2)(x-1)^2+\beta_2-\beta_2^N}{1-\beta_2^N},\quad \bar v_n=\beta_2^n\bar v_0+1-\beta_2^n.\]
It depends on $p$ only through $x$, and not on $a$.
\item[(ii)] If $N\ge3$, then $\Delta(N)<0$: at the unbiased value $p=q_i$, the logit decreases.
\item[(iii)] $\Delta$ has a unique zero $x^\star\in(N,\infty)$, which is stable. The fixed point $p^\star=a/x^\star$ satisfies $p^\star<q_i$, and $p^\star/q_i=N/x^\star$ depends only on $N$ and $\beta_2$; in particular it does not depend on $\eta$.
\item[(iv)] If $N\to\infty$ with $\beta_2=1-\kappa/N$ for a fixed $\kappa>0$, then $p^\star/q_i\to\rho(\kappa):=\dfrac{\kappa}{2(e^{\kappa/2}-1)}$.
\end{enumerate}
\end{theorem}

The function $\rho$ decreases from $1$ to $0$, with $\rho(\kappa)=1-\kappa/4+O(\kappa^2)$ and $\log\rho(\kappa)=\log(\kappa/2)-\kappa/2+o(1)$. For comparison, $\rho(0.1)=0.975$, $\rho(1)=0.771$, $\rho(2)=0.582$, $\rho(5)=0.224$, and $\rho(10)=0.034$. The parameter $\kappa=N(1-\beta_2)$ compares the mean number of steps between occurrences with the second-moment time constant $1/(1-\beta_2)$; equivalently, $\kappa=(1-\beta_2)/(q_iB)$. Three implications are particularly useful. \emph{The learning rate does not change the fixed point:} upward and downward steps both scale with $\eta$, so reducing it only slows the approach to $p^\star$. \emph{For a fixed token, only $(1-\beta_2)/B$ matters:} halving $B$ has the same effect as doubling $1-\beta_2$. \emph{The bias becomes large once the token is observed less often than once per time constant,} i.e. when $\kappa>1$. Appendix~\ref{app:proofs} derives the periodic orbit of $v$, uses that $\bar v_n$ is a convex combination of $\bar v_0$ and $1$, and shows that $\Delta$ is increasing in $x$. An independent numerical computation of the periodic orbit agrees with the closed form to relative error \ThmCheckMaxRelErr, and at $N=4000$ the finite-$N$ ratio differs from $\rho(\kappa)$ by relative error \ThmLimitMaxRelErr.

\begin{proposition}[Unbiased and unstable cases]\label{prop:others}
Under (A1)--(A2):
\textbf{(a)} SGD with a constant step changes $\log p_i$ by $-\eta p(N-x)$ over one period; its fixed point is $p=q_i$, and it is stable.
\textbf{(b)} For AMSGrad with $\beta_1=0$, the running maximum $\hat v_t$ converges to some $\hat v_\infty>0$, and the change over one period converges to $-\eta\hat v_\infty^{-1/2}p(N-x)$; the fixed point is again $p=q_i$.
\textbf{(c)} Sign descent changes the logit by $-\eta(N-2)$ over one period for every $p$; for $N\ge3$ there is no fixed point.
\end{proposition}

For Coupled Adam, token $i$ contributes only a fraction $1/V$ of the vocabulary-wide second-moment average. When that average is constant over a period, the method reduces to the behavior in (a). Theorem~\ref{thm:rms} does not address $\beta_1>0$, random arrivals, or the rows of $W$, where the gradients also depend on the contexts in which the token occurs. Section~\ref{sec:exp} tests all three effects.

\section{Experiments}\label{sec:exp}
Every number and figure in this section comes from \texttt{run\_experiments.py}, provided as supplementary material; all experiments run in under two hours on a CPU. Appendix~\ref{app:details} gives the complete configuration.

\subsection{Conservation across optimizers}
We train a softmax regression with $V=512$ classes, 128 of which never occur, on features with a nonzero mean ($d=32$, $B=64$, 300 steps, float64). Table~\ref{tab:cons} records the changes in $\wbar$ and $\bar b$. Every optimizer covered by Proposition~\ref{prop:cons} preserves both quantities up to rounding error, with largest change \ConsMaxConserving; each of the remaining optimizers changes them, with smallest change \ConsMinBreaking.

\begin{table}[t]
\centering\small
\caption{Change of the mean output embedding and mean bias after training a softmax regression in float64. ``Conserved'' is the prediction of Proposition~\ref{prop:cons}; Muon is applied to $W$, with heavy-ball momentum on $b$.}\label{tab:cons}
\begin{tabular}{llccc}
\toprule
Optimizer & Update & Conserved (Prop.~1) & $\|\Delta\bar{\mathbf w}\|_2$ & $|\Delta\bar b|$\\
\midrule
SGD & linear & yes & $5.4\times 10^{-17}$ & $3.5\times 10^{-18}$\\
Heavy ball & linear & yes & $5.7\times 10^{-17}$ & $3.5\times 10^{-18}$\\
Nesterov & linear & yes & $9.0\times 10^{-17}$ & $3.5\times 10^{-18}$\\
Shampoo & Kronecker & yes & $1.8\times 10^{-16}$ & $1.6\times 10^{-15}$\\
Muon (heavy ball on $b$) & orthogonalized & yes & $7.5\times 10^{-17}$ & $6.9\times 10^{-18}$\\
Coupled Adam & shared scaling & yes & $5.0\times 10^{-17}$ & $0$\\
Adam & coordinate-wise & no & $6.6\times 10^{-1}$ & $1.4\times 10^{-1}$\\
RMSProp & coordinate-wise & no & $6.4\times 10^{-1}$ & $1.4\times 10^{-1}$\\
AMSGrad & coordinate-wise & no & $5.1\times 10^{-1}$ & $1.1\times 10^{-1}$\\
Adafactor & factored & no & $5.3\times 10^{-1}$ & $1.4\times 10^{-1}$\\
Lion & sign & no & $3.6\times 10^{-1}$ & $7.1\times 10^{-2}$\\
Sign descent & sign & no & $1.4\times 10^{0}$ & $2.6\times 10^{-1}$\\
\bottomrule
\end{tabular}

\end{table}

\subsection{A single token and the unigram model}
\textbf{Single token.} We iterate the recursion from Section~\ref{sec:rmsprop} with $p=e^{b_i}$, a fixed log-partition, and $a=10^{-3}$, using 30 values of $N$ between 3 and 1000, $\beta_2=0.99$, and $\eta=3\times10^{-3}$. We initialize at the unbiased value and average $\log p$ over the last third of $4\times10^5$ steps. Unlike (A2), $p$ is allowed to move within each period. For periodic arrivals, the measured fixed point agrees with Theorem~\ref{thm:rms}(iii), with largest deviation in $\log p$ equal to \STPerMaxDev (Figure~\ref{fig:unigram}a). Adam with $\beta_1=0.9$ differs by at most \STAdamDevMax, so momentum changes the fixed point little. Random arrivals behave differently: the token appears independently at each step with probability $1/N$; for $\kappa<1$ the results remain close to the periodic prediction (mean deviation \STRandDevLow), whereas for $\kappa\ge2$ the logit averaged over the last third of $4\times10^5$ steps lies well below it (mean deviation \STRandDevHigh). The periodic formula therefore understates the bias when $\kappa$ is large: the downward push over a gap grows faster than linearly in the gap length, and geometric gaps sometimes extend far beyond their mean.

\textbf{Unigram model.} We use $V=5000$ tokens with Zipf frequencies $q_i\propto i^{-1.2}$, initialize at the population optimum $b=\log q$, and run $3\times10^5$ steps. Adaptive methods use learning rate $4\times10^{-3}$ and $\epsilon=10^{-12}$. The three pairs $(\beta_2,B)\in\{(0.95,500),(0.99,100),(0.999,100)\}$ give $\kappa_i=(1-\beta_2)/\pi_i$ ranging from $10^{-2}$ to about 13. Figure~\ref{fig:unigram}b plots the average $\log(p_i/q_i)$ over the second half of training against $\kappa_i$. Arrivals are random, and the pattern follows the single-token experiment. Under both RMSProp and Adam, tokens with $\kappa>1$ receive less probability than their frequency, with the bias growing as $\kappa$ increases. For $0.2\lesssim\kappa\lesssim1$ it is slightly positive (below $0.2$ it is smaller), while for large $\kappa$ it exceeds $\log\rho(\kappa)$. Over bins with $0.1\le\kappa\le10$, the largest deviation from $\log\rho(\kappa)$ is \UniRMSMaxDev{} for RMSProp and \UniAdamMaxDev{} for Adam; near $\kappa=5$, RMSProp with $\beta_2=0.99$ gives \UniRMSKfive{} against $\log\rho(5)=\text{\UniThKfive}$. Theorem~\ref{thm:rms} therefore predicts the sign of the bias for $\kappa>1$ and its growth with $\kappa$, but the closed form is quantitative only for periodic arrivals or small $\kappa$. For $\kappa\lesssim1$, the three $(\beta_2,B)$ pairs lie on one curve; for larger $\kappa$ they separate. At the same $\kappa$, $(0.99,100)$ is more biased than $(0.95,500)$ (about $-6.4$ versus $-4.7$ near $\kappa=7.5$). With random arrivals, $\kappa$ therefore no longer determines the bias by itself once it exceeds 1. Figure~\ref{fig:unigram}c gives the controls: the mean of $|\log(p_i/q_i)|$ over tokens with $\kappa_i\ge1$ is \UniSGDAbs{} for SGD, \UniAMSAbs{} for AMSGrad, and \UniCoupledAbs{} for Coupled Adam. Under sign descent, \UniSignFracBelow{} of tokens with $\pi_i<1/2$ end with $\log(p_i/q_i)<-5$, compared with \UniSignFracAbove{} of tokens with $\pi_i>1/2$, consistent with Theorem~\ref{thm:sign}.

\begin{figure}[t]
\centering
\figorbox{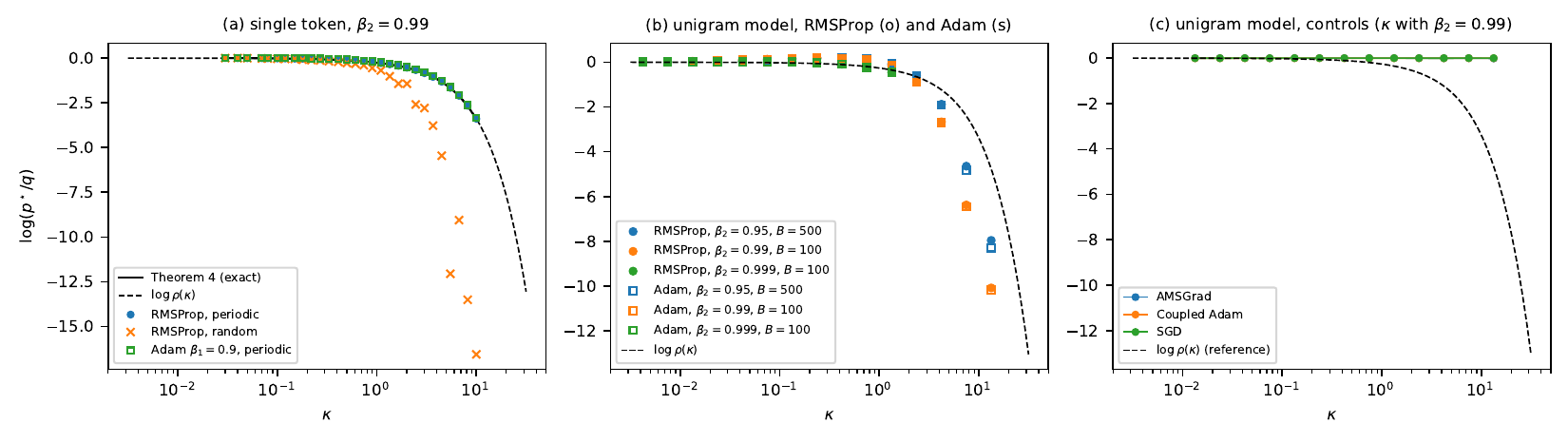}{\linewidth}
\caption{Fixed points for rare tokens. (a) Single token, $\beta_2=0.99$: measured $\log(p^\star/q)$ for RMSProp with periodic and random arrivals and for Adam ($\beta_1=0.9$) with periodic arrivals; solid: exact fixed point of Theorem~\ref{thm:rms}(iii); dashed: $\log\rho(\kappa)$. (b) Unigram model, $V=5000$: time average of $\log(p_i/q_i)$, binned by $\kappa_i=(1-\beta_2)/\pi_i$, for three pairs $(\beta_2,B)$; circles: RMSProp, squares: Adam. (c) The same for SGD, AMSGrad and Coupled Adam ($\kappa_i$ computed with $\beta_2=0.99$).}\label{fig:unigram}
\end{figure}

\subsection{A small language model}\label{sec:lm}
\textbf{Setup.} The data come from a first-order Markov chain on 2048 tokens, with $P(y\mid x)\propto q_y\exp(3\,u_x^\top u_y)$, Zipf weights $q_y\propto y^{-1.1}$, and random unit vectors $u_x\in\R^{16}$. The vocabulary contains $V=4096$ entries, half of which never occur. We draw $5.2\times10^5$ training pairs $(x,y)$. The model uses an embedding of width 64, a residual MLP block with LayerNorm and hidden width 256, a final LayerNorm, and an untied output layer with bias. Training runs for $2\times10^4$ steps with $B=256$, a 200-step warmup, and then a constant learning rate. Since the generating distribution is known, we can measure the KL divergence from that distribution to the model, averaged over contexts. The output bias also gives a stationarity condition: every stationary point satisfies $\bar p_y=f_y$, where $\bar p$ is the prediction averaged over training contexts and $f$ is the empirical token distribution. Thus $r_y=\log(\bar p_y/f_y)$ records either incomplete optimization or a bias introduced by the optimizer. We plot it against $e_y=f_yB$, the expected number of occurrences of $y$ in a minibatch.

\textbf{Results.} With $\beta_2=0.95$, Adam, RMSProp and AdamW give rare tokens ($e_y<0.05$) less probability than their frequency (mean $r_y$: \LMRareAdamA, \LMRareRMS{} and \LMRareAdamW). Raising $\beta_2$ to $0.999$ reduces the bias to \LMRareAdamB. AMSGrad and Coupled Adam, which the theory predicts to be unbiased, give \LMRareAMS{} and \LMRareCoupled. SGD gives \LMRareSGD; since SGD has no bias at its fixed point, this value measures incomplete optimization and sets the resolution of the measurement. For $\beta_2=0.95$ the bias is larger than the unigram prediction for the same tokens (\LMThA), as in the unigram model with random arrivals; for $\beta_2=0.999$ both the prediction (\LMThB) and the measurement are below this resolution. The bias comes with a worse fit to the generating distribution. KL and test cross-entropy (Table~\ref{tab:lm}) follow nearly the same ordering: SGD has the lowest KL, AMSGrad and Coupled Adam come next, and the three biased methods with $\beta_2=0.95$ have the highest. We do not separate how much of this difference comes from rare tokens. Figure~\ref{fig:lm}b shows the mean logit. It stays at its initial value under SGD (\LMDriftSGD) and Coupled Adam (\LMDriftCoupled), as Proposition~\ref{prop:cons} requires, and moves under Adam (\LMDriftAdamA{} with $\beta_2=0.95$, \LMDriftAdamB{} with $0.999$), RMSProp (\LMDriftRMS) and AMSGrad (\LMDriftAMS). Weight decay with $\lambda=0.1$ reduced the change to \LMDriftAdamW{} but did not stop it.

\begin{figure}[t]
\centering
\figorbox{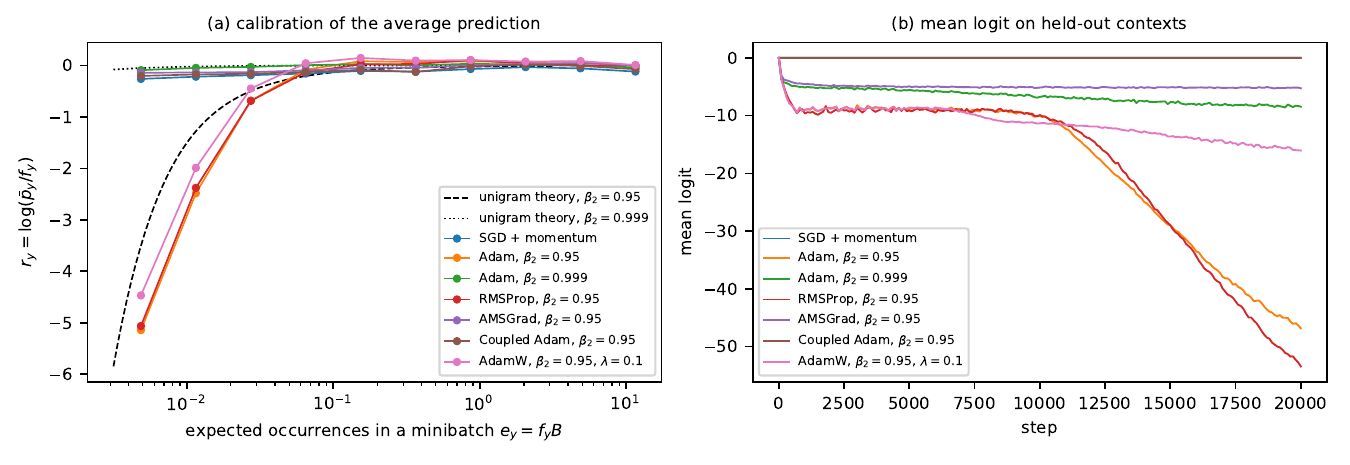}{0.85\linewidth}
\caption{Small language model. (a) $r_y=\log(\bar p_y/f_y)$, binned by the expected number of occurrences in a minibatch; dashed and dotted: $\log\rho(\kappa)$ from the unigram theory for $\beta_2=0.95$ and $0.999$. (b) Mean logit over the vocabulary on held-out contexts.}\label{fig:lm}
\end{figure}

\begin{table}[t]
\centering\small
\caption{Small language model after training. KL: divergence from the generating distribution to the model, averaged over contexts. $\bar r$: mean of $r_y$ over tokens in the stated range of $e_y$. Mean logit: over the vocabulary and 2048 held-out contexts. Unused $\log p$: average log-probability of the entries that never occur.}\label{tab:lm}
\begin{tabular}{lrrrrrr}
\toprule
Optimizer & Test CE & KL & $\bar r$, $e_y{<}0.05$ & $\bar r$, $e_y{\ge}1$ & Mean logit & Unused $\log p$\\
\midrule
SGD + momentum & $3.467$ & $0.0940$ & $-0.22$ & $-0.06$ & $0.00$ & $-13.5$\\
Adam, $\beta_2{=}0.95$ & $3.603$ & $0.2278$ & $-2.72$ & $0.04$ & $-46.83$ & $-27.4$\\
Adam, $\beta_2{=}0.999$ & $3.537$ & $0.1586$ & $-0.06$ & $-0.01$ & $-8.45$ & $-21.7$\\
RMSProp, $\beta_2{=}0.95$ & $3.594$ & $0.2190$ & $-2.65$ & $0.03$ & $-53.42$ & $-26.6$\\
AMSGrad, $\beta_2{=}0.95$ & $3.500$ & $0.1241$ & $-0.14$ & $0.01$ & $-5.25$ & $-16.4$\\
Coupled Adam, $\beta_2{=}0.95$ & $3.499$ & $0.1238$ & $-0.18$ & $-0.02$ & $-0.00$ & $-13.2$\\
AdamW, $\beta_2{=}0.95$, $\lambda{=}0.1$ & $3.558$ & $0.1828$ & $-2.25$ & $0.06$ & $-16.04$ & $-20.5$\\
\bottomrule
\end{tabular}

\end{table}

\section{Related work}
\textbf{Symmetries and conservation laws.} \citet{kunin2021neural} identified quantities preserved by gradient flow under translation, scale, and rescale symmetries, including the translation symmetry of the softmax, and analyzed how weight decay, momentum, stochasticity, and finite step sizes can break them; see also \citet{tanaka2021noether,zhao2023symmetries}. Proposition~\ref{prop:cons} places Kronecker-factored and orthogonalized methods \citep{gupta2018shampoo,jordan2024muon,bernstein2024old} in the conserving class, while Proposition~\ref{prop:cov} gives the exact per-step change for coordinate-wise methods that violate the conservation law.

\textbf{Common shift of output embeddings and output logit divergence.} \citet{gao2019representation} observed degeneration of embeddings toward a narrow cone, while \citet{bis2021too} reported a common shift. \citet{stollenwerk2025coupled} attributed that shift to Adam's second-moment estimate and proposed Coupled Adam. Later, \citet{stollenwerk2026oec} related the mean output embedding to output-logit divergence \citep{chowdhery2023palm,wortsman2024small} and proposed output embedding centering. \citet{liu2026grokking} showed that low-precision rounding can break the zero-sum gradient property in finite-precision training. All of this work concerns the common shift. The present analysis instead focuses on the fixed point of each token, which centering leaves unchanged.

\textbf{Adaptive methods and rare classes.} \citet{kunstner2024heavy} found that, under heavy-tailed class imbalance, gradient descent makes slower progress on rare classes whereas sign descent does not; \citet{balles2018dissecting,kunstner2023noise} discuss the connection between Adam and sign descent. The result here is complementary: the faster progress of adaptive normalization can come with a shifted fixed point. \citet{land2024fishing} use output embeddings to detect under-trained tokens.

\textbf{Non-convergence of Adam.} \citet{reddi2018convergence} constructed a one-dimensional example with a large gradient every few steps and smaller gradients of the opposite sign between them, where Adam converges to the wrong point, and introduced AMSGrad as a remedy. The gradient pattern of a rare token's logit has the same form, making rare tokens an instance of that construction; Theorem~\ref{thm:rms} gives the corresponding fixed point in closed form. \citet{zhang2022adam} established convergence when $\beta_2$ is sufficiently close to $1$, while \citet{kim2026second} showed that second-moment memory slows progress in the presence of rare large gradients. Adagrad \citep{duchi2011adaptive}, whose second moment never decays, appears here as the limit $\kappa\to0$.

\textbf{Long-tailed classification and calibration.} Logit adjustment \citep{menon2021longtail} explicitly corrects logits for rare classes, while \citet{guo2017calibration} study calibration in modern neural networks. The present analysis identifies an additional optimizer-dependent mechanism that can shift the probabilities assigned to rare classes.

\section{Discussion and limitations}
The theory relies on several strong simplifications: a single bias coordinate, periodic arrivals, $\beta_1=0$, $\epsilon=0$, and assumption (A2). The experiments relax the first three assumptions, but only on small synthetic problems. In the small language model, the optimizers with biased fixed points also produce a higher KL divergence from the generating distribution. Whether the same effect is measurable in large language models, and how much it influences downstream behavior such as the generation of rare words, remains open. During pretraining with millions of tokens per batch, $\kappa>1$ occurs only for extremely rare tokens; in fine-tuning and other small-batch regimes with $\beta_2=0.95$, a substantial fraction of a vocabulary can satisfy $\kappa>1$. For the output layer, the analysis suggests several ways to reduce the bias: increasing $\beta_2$, increasing the batch size, using a running maximum of the second moment as in AMSGrad, sharing the second moment across the vocabulary as in Coupled Adam, or using SGD for the output bias. A positive $\epsilon$ also suppresses the effect once $\sqrt v$ falls below $\epsilon$. Finally, with tied embeddings \citep{press2017using}, the common shift from Section~\ref{sec:mean} now affects the loss because it also shifts the input embeddings.

\bibliography{refs}

@inproceedings{kingma2015adam, title={Adam: A Method for Stochastic Optimization}, author={Kingma, Diederik P. and Ba, Jimmy}, booktitle={International Conference on Learning Representations}, year={2015}}

@inproceedings{reddi2018convergence, title={On the Convergence of {Adam} and Beyond}, author={Reddi, Sashank J. and Kale, Satyen and Kumar, Sanjiv}, booktitle={International Conference on Learning Representations}, year={2018}}

@inproceedings{loshchilov2019decoupled, title={Decoupled Weight Decay Regularization}, author={Loshchilov, Ilya and Hutter, Frank}, booktitle={International Conference on Learning Representations}, year={2019}}

@inproceedings{kunin2021neural, title={Neural Mechanics: Symmetry and Broken Conservation Laws in Deep Learning Dynamics}, author={Kunin, Daniel and Sagastuy-Brena, Javier and Ganguli, Surya and Yamins, Daniel L. K. and Tanaka, Hidenori}, booktitle={International Conference on Learning Representations}, year={2021}}

@inproceedings{tanaka2021noether, title={Noether's Learning Dynamics: Role of Symmetry Breaking in Neural Networks}, author={Tanaka, Hidenori and Kunin, Daniel}, booktitle={Advances in Neural Information Processing Systems}, year={2021}}

@inproceedings{zhao2023symmetries, title={Symmetries, Flat Minima, and the Conserved Quantities of Gradient Flow}, author={Zhao, Bo and Ganev, Iordan and Walters, Robin and Yu, Rose and Dehmamy, Nima}, booktitle={International Conference on Learning Representations}, year={2023}}

@inproceedings{stollenwerk2025coupled, title={Better Embeddings with Coupled {Adam}}, author={Stollenwerk, Felix and Stollenwerk, Tobias}, booktitle={Proceedings of the 63rd Annual Meeting of the Association for Computational Linguistics (Volume 1: Long Papers)}, pages={27219--27236}, year={2025}}

@article{stollenwerk2026oec, title={Output Embedding Centering for Stable {LLM} Pretraining}, author={Stollenwerk, Felix and Lokrantz, Anna and Hertzberg, Niclas}, journal={arXiv preprint arXiv:2601.02031}, year={2026}}

@article{liu2026grokking, title={Grokking or Glitching? How Low-Precision Drives Slingshot Loss Spikes}, author={Liu, Hanqing and Cao, Jianjun and Li, Yuanze and Zhou, Zijian}, journal={arXiv preprint arXiv:2605.06152}, year={2026}}

@inproceedings{wortsman2024small, title={Small-scale Proxies for Large-scale Transformer Training Instabilities}, author={Wortsman, Mitchell and Liu, Peter J. and Xiao, Lechao and Everett, Katie and Alemi, Alex and Adlam, Ben and Co-Reyes, John D. and Gur, Izzeddin and Kumar, Abhishek and Novak, Roman and Pennington, Jeffrey and Sohl-Dickstein, Jascha and Xu, Kelvin and Lee, Jaehoon and Gilmer, Justin and Kornblith, Simon}, booktitle={International Conference on Learning Representations}, year={2024}}

@article{chowdhery2023palm, title={{PaLM}: Scaling Language Modeling with Pathways}, author={Chowdhery, Aakanksha and Narang, Sharan and Devlin, Jacob and others}, journal={Journal of Machine Learning Research}, volume={24}, number={240}, pages={1--113}, year={2023}}

@inproceedings{gao2019representation, title={Representation Degeneration Problem in Training Natural Language Generation Models}, author={Gao, Jun and He, Di and Tan, Xu and Qin, Tao and Wang, Liwei and Liu, Tie-Yan}, booktitle={International Conference on Learning Representations}, year={2019}}

@inproceedings{bis2021too, title={Too Much in Common: Shifting of Embeddings in Transformer Language Models and its Implications}, author={Bi{\'s}, Daniel and Podkorytov, Maksim and Liu, Xiuwen}, booktitle={Proceedings of the 2021 Conference of the North American Chapter of the Association for Computational Linguistics: Human Language Technologies}, year={2021}}

@inproceedings{kunstner2024heavy, title={Heavy-Tailed Class Imbalance and Why {Adam} Outperforms Gradient Descent on Language Models}, author={Kunstner, Frederik and Milligan, Alan and Yadav, Robin and Schmidt, Mark and Bietti, Alberto}, booktitle={Advances in Neural Information Processing Systems}, year={2024}}

@inproceedings{kunstner2023noise, title={Noise Is Not the Main Factor Behind the Gap Between {SGD} and {Adam} on Transformers, But Sign Descent Might Be}, author={Kunstner, Frederik and Chen, Jacques and Lavington, Jonathan Wilder and Schmidt, Mark}, booktitle={International Conference on Learning Representations}, year={2023}}

@inproceedings{land2024fishing, title={Fishing for {Magikarp}: Automatically Detecting Under-trained Tokens in Large Language Models}, author={Land, Sander and Bartolo, Max}, booktitle={Proceedings of the 2024 Conference on Empirical Methods in Natural Language Processing}, year={2024}}

@inproceedings{gupta2018shampoo, title={Shampoo: Preconditioned Stochastic Tensor Optimization}, author={Gupta, Vineet and Koren, Tomer and Singer, Yoram}, booktitle={International Conference on Machine Learning}, year={2018}}

@misc{jordan2024muon, title={Muon: An Optimizer for Hidden Layers in Neural Networks}, author={Jordan, Keller and Jin, Yuchen and Boza, Vlado and You, Jiacheng and Cesista, Franz and Newhouse, Laker and Bernstein, Jeremy}, year={2024}, howpublished={\url{https://kellerjordan.github.io/posts/muon/}}}

@article{bernstein2024old, title={Old Optimizer, New Norm: An Anthology}, author={Bernstein, Jeremy and Newhouse, Laker}, journal={arXiv preprint arXiv:2409.20325}, year={2024}}

@inproceedings{shazeer2018adafactor, title={Adafactor: Adaptive Learning Rates with Sublinear Memory Cost}, author={Shazeer, Noam and Stern, Mitchell}, booktitle={International Conference on Machine Learning}, year={2018}}

@inproceedings{chen2023lion, title={Symbolic Discovery of Optimization Algorithms}, author={Chen, Xiangning and Liang, Chen and Huang, Da and Real, Esteban and others}, booktitle={Advances in Neural Information Processing Systems}, year={2023}}

@inproceedings{bernstein2018signsgd, title={sign{SGD}: Compressed Optimisation for Non-Convex Problems}, author={Bernstein, Jeremy and Wang, Yu-Xiang and Azizzadenesheli, Kamyar and Anandkumar, Animashree}, booktitle={International Conference on Machine Learning}, year={2018}}

@article{duchi2011adaptive, title={Adaptive Subgradient Methods for Online Learning and Stochastic Optimization}, author={Duchi, John and Hazan, Elad and Singer, Yoram}, journal={Journal of Machine Learning Research}, volume={12}, pages={2121--2159}, year={2011}}

@misc{tieleman2012rmsprop, title={Lecture 6.5---{RMSProp}: Divide the Gradient by a Running Average of its Recent Magnitude}, author={Tieleman, Tijmen and Hinton, Geoffrey}, howpublished={COURSERA: Neural Networks for Machine Learning}, year={2012}}

@inproceedings{zhang2022adam, title={Adam Can Converge Without Any Modification on Update Rules}, author={Zhang, Yushun and Chen, Congliang and Shi, Naichen and Sun, Ruoyu and Luo, Zhi-Quan}, booktitle={Advances in Neural Information Processing Systems}, year={2022}}

@article{kim2026second, title={Second-Moment Memory in Coordinatewise {Adam}}, author={Kim, Jeonseong}, journal={arXiv preprint arXiv:2608.15824}, year={2026}}

@inproceedings{balles2018dissecting, title={Dissecting {Adam}: The Sign, Magnitude and Variance of Stochastic Gradients}, author={Balles, Lukas and Hennig, Philipp}, booktitle={International Conference on Machine Learning}, year={2018}}

@inproceedings{menon2021longtail, title={Long-tail Learning via Logit Adjustment}, author={Menon, Aditya Krishna and Jayasumana, Sadeep and Rawat, Ankit Singh and Jain, Himanshu and Veit, Andreas and Kumar, Sanjiv}, booktitle={International Conference on Learning Representations}, year={2021}}

@inproceedings{guo2017calibration, title={On Calibration of Modern Neural Networks}, author={Guo, Chuan and Pleiss, Geoff and Sun, Yu and Weinberger, Kilian Q.}, booktitle={International Conference on Machine Learning}, year={2017}}

@article{piantadosi2014zipf, title={Zipf's Word Frequency Law in Natural Language: A Critical Review and Future Directions}, author={Piantadosi, Steven T.}, journal={Psychonomic Bulletin \& Review}, volume={21}, number={5}, pages={1112--1130}, year={2014}}

@inproceedings{press2017using, title={Using the Output Embedding to Improve Language Models}, author={Press, Ofir and Wolf, Lior}, booktitle={Proceedings of the 15th Conference of the European Chapter of the Association for Computational Linguistics: Volume 2, Short Papers}, pages={157--163}, year={2017}}
\bibliographystyle{iclr2027_conference}

\appendix
\section{Proofs}\label{app:proofs}

\textbf{Proof of Proposition~\ref{prop:cons}.} By (F1), $\one^\top G_s=0$ for all $s$. (a) $\one^\top U_t=\sum_s\alpha_{t,s}\one^\top G_s=0$. (b) $\one^\top A_tM_tR_t=(A_t^\top\one)^\top M_tR_t=\alpha\,\one^\top M_tR_t=0$. For Shampoo, $L_t\one=\epsilon\one+\sum_s\gamma_{t,s}G_s(G_s^\top\one)=\epsilon\one$, so $\one$ is an eigenvector of the symmetric matrix $L_t$ and $A_t\one=L_t^{-1/4}\one=\epsilon^{-1/4}\one$. (c) $\one^\top M_t\psi_t(M_t^\top M_t)=0$. The Newton--Schulz step $X\mapsto aX+bXX^\top X+c(XX^\top)^2X=X(aI+bX^\top X+c(X^\top X)^2)$ maps a matrix of the form $M\phi(M^\top M)$ to another matrix of this form, because $X^\top X$ is then a function of $M^\top M$; by induction the output is $M\psi(M^\top M)$, and transposing to work on the smaller side does not change this. (d) $\sum_iU_{t,ij}=d_{t,j}\sum_iM_{t,ij}=0$. Finally $\wbar_{t+1}=W_{t+1}^\top\one/V=(1-\eta_t\lambda)\wbar_t-\eta_tU_t^\top\one/V$. The bias is the case $h\equiv1$. \hfill$\square$

\textbf{Proof of Proposition~\ref{prop:cov}.} Since $\sum_iM_{t,ij}=0$, $\frac1V\sum_id_{t,ij}M_{t,ij}=\frac1V\sum_i(d_{t,ij}-\bar d_{t,j})M_{t,ij}=\Cov_i(d_{t,ij},M_{t,ij})$, where the last step uses $\bar M_{t,j}=0$. For sign updates, $\frac1V\sum_i\sign(S_{t,ij})=(n^+_{t,j}-n^-_{t,j})/V$. \hfill$\square$

\textbf{Proof of Theorem~\ref{thm:sign}.} If $c_{t,i}=0$ then $g_{t,i}=p_i(b_t)>0$ and $b_{t+1,i}-b_{t,i}=-\eta$; otherwise $b_{t+1,i}-b_{t,i}\le\eta$. As $c_t$ is independent of $\mathcal F_t$, $\E[b_{t+1,i}-b_{t,i}\mid\mathcal F_t]\le-\eta(1-\pi_i)+\eta\pi_i$. The differences $D_t=(b_{t+1,i}-b_{t,i})-\E[b_{t+1,i}-b_{t,i}\mid\mathcal F_t]$ form a martingale difference sequence bounded by $2\eta$, so by the Azuma--Hoeffding inequality and the Borel--Cantelli lemma $\frac1t\sum_{s<t}D_s\to0$ almost surely, which gives the $\limsup$ bound. If $q_i=0$, then $c_{t,i}=0$ at every step. For the mean, average the bound over $i$ and use $\sum_i\pi_i=\E[\#\{i:c_{t,i}\ge1\}]\le\min(B,V)$. \hfill$\square$

\textbf{Proof of Theorem~\ref{thm:rms}.} (i) Over one period, $v$ is updated by an affine map with slope $\beta_2^N<1$, which has a unique fixed point that attracts every initial value; this gives the periodic sequence. All gradients are multiples of $p$ ($g_0=-(x-1)p$ and $g_n=p$ for $1\le n\le N-1$), so write $v_n=p^2\bar v_n$. For $1\le n\le N-1$, $\bar v_n=\beta_2\bar v_{n-1}+(1-\beta_2)$, hence $\bar v_n=\beta_2^n\bar v_0+1-\beta_2^n$; and $\bar v_0=\beta_2\bar v_{N-1}+(1-\beta_2)(x-1)^2=\beta_2^N\bar v_0+\beta_2-\beta_2^N+(1-\beta_2)(x-1)^2$, which gives the formula for $\bar v_0$. Summing $-\eta g_n/\sqrt{v_n}$ over one period gives $\Delta(x)$; $p$ and $a$ cancel.

(ii) At $x=N$, $\bar v_0-1=(1-\beta_2)\big((N-1)^2-1\big)/(1-\beta_2^N)>0$ for $N\ge3$. For $n\ge1$, $\bar v_n$ is a convex combination of $\bar v_0$ and $1$ with positive weight on $1$, so $\bar v_n<\bar v_0$. Hence $\Delta(N)=-\eta\sum_{n=1}^{N-1}(\bar v_n^{-1/2}-\bar v_0^{-1/2})<0$.

(iii) Write $\Delta=\eta F$ with $F(x)=(x-1)\bar v_0^{-1/2}-\sum_{n=1}^{N-1}\bar v_n^{-1/2}$ on $x>1$. With $c=\beta_2-\beta_2^N>0$, the first term equals $\sqrt{1-\beta_2^N}\,(x-1)/\sqrt{(1-\beta_2)(x-1)^2+c}$, whose derivative $\sqrt{1-\beta_2^N}\,c\,((1-\beta_2)(x-1)^2+c)^{-3/2}$ is positive. Each $\bar v_n$ increases with $x$ through $\bar v_0$, so each $-\bar v_n^{-1/2}$ increases. Thus $F$ is strictly increasing. By (ii) $F(N)<0$ (for $N=2$, $\bar v_0=\bar v_1=1$ and $x^\star=N$), and $F(x)\to\sqrt{(1-\beta_2^N)/(1-\beta_2)}>0$ as $x\to\infty$ because every $\bar v_n\to\infty$. So $F$ has a unique zero $x^\star>N$, and $p^\star=a/x^\star<a/N=q_i$. The zero depends only on $(N,\beta_2)$. Stability: if $p>p^\star$ then $x<x^\star$, $\Delta<0$ and $\log p_i$ decreases, and conversely.

(iv) Let $\beta_2=1-\kappa/N$ and $x=\xi N$. The zero condition is $\xi N-1=\sum_{n=1}^{N-1}\big(\beta_2^n+(1-\beta_2^n)/\bar v_0\big)^{-1/2}$. For $\xi$ in a compact subset of $(0,\infty)$, $\bar v_0\ge(\kappa/N)(\xi N-1)^2$ grows linearly in $N$, so $(1-\beta_2^n)/\bar v_0=O(1/N)$ uniformly, and $\beta_2^n=e^{-\kappa n/N}(1+O(1/N))$ uniformly in $n\le N$. Hence the right-hand side divided by $N$ converges uniformly on compacts to $\int_0^1e^{\kappa s/2}\,ds=\frac2\kappa(e^{\kappa/2}-1)$, while the left-hand side divided by $N$ converges to $\xi$. The functions $F_N(\xi)=(\text{LHS}-\text{RHS})/N$ are increasing by (iii), satisfy $F_N(1)<0$, and converge uniformly on compacts to $\xi-\frac2\kappa(e^{\kappa/2}-1)$, which has a unique simple zero $\xi^\star$. Hence the zeros of $F_N$ converge to $\xi^\star$, and $p^\star/q_i=N/x^\star=1/\xi_N\to1/\xi^\star=\rho(\kappa)$. \hfill$\square$

\textbf{Proof of Proposition~\ref{prop:others}.} (a) Over one period, $\sum_ng_n=Np-a=p(N-x)$. (b) The running maximum is non-decreasing and bounded, since $|g_t|\le a$, so it converges to $\hat v_\infty\ge\max_nv_n>0$; the change over one period then converges to $-\eta\hat v_\infty^{-1/2}\sum_ng_n$. (c) Since $p<a$, the arrival step has $\sign(g_0)=-1$ and the other $N-1$ steps have sign $+1$. \hfill$\square$

\section{Experimental details}\label{app:details}
\textbf{Conservation (Table~\ref{tab:cons}).} Softmax regression with $V=512$ classes, of which 384 occur with Zipf weights $\propto i^{-1.1}$; features $h\sim\mathcal N(\mu,I_{32})$ with a random mean $\mu$; labels drawn from a random softmax teacher. $B=64$, 300 steps, float64, $W_0\sim\mathcal N(0,0.02^2)$, $b_0=0$. Learning rates: SGD 0.5; heavy ball and Nesterov 0.05 (momentum 0.9); Shampoo 0.02 ($\epsilon=10^{-4}$, accumulated statistics); Muon 0.02 (momentum 0.95, five quintic Newton--Schulz steps, scale $\sqrt{\max(1,V/d)}$); Adam, RMSProp, AMSGrad, Coupled Adam and Adafactor $3\times10^{-3}$ ($\beta_2=0.999$); Lion $3\times10^{-4}$ ($\beta_1=0.9$, $\beta_2=0.99$); sign descent $10^{-3}$. Changes are reported, not losses; the learning rates only set the scale of the changes.

\textbf{Single token.} $\beta_2=0.99$, $a=10^{-3}$, $\eta=3\times10^{-3}$, $\epsilon=0$, no bias correction, $4\times10^5$ steps, initial logit $\log(a/N)$; average of the logit over the last third.

\textbf{Unigram model.} $V=5000$, $q_i\propto i^{-1.2}$, initial $b=\log q$, $3\times10^5$ steps; learning rate $4\times10^{-3}$ for RMSProp, Adam, AMSGrad and Coupled Adam, 1 for SGD, $10^{-3}$ for sign descent; $\epsilon=10^{-12}$; $\log p$ recorded every 20 steps in the second half. All runs with the same $B$ see the same minibatches.

\textbf{Language model.} Markov chain over 2048 tokens as in Section~\ref{sec:lm}; 8192 chains of length 64 after a burn-in of 50 steps for training, 1024 chains of length 32 for evaluation. Model: embedding (width 64), LayerNorm, MLP $64\to256\to64$ with GELU and a residual connection, final LayerNorm, linear output layer with bias ($W_0\sim\mathcal N(0,0.02^2)$, $b_0=0$), untied. $B=256$ pairs, $2\times10^4$ steps, linear warmup over 200 steps, then constant. SGD: learning rate 0.02, momentum 0.9. Adaptive methods: learning rate $3\times10^{-3}$, $\beta_1=0.9$ (0 for RMSProp), $\epsilon=10^{-8}$; Coupled Adam shares the second moment across the vocabulary for $W$ and $b$ of the output layer; AdamW uses decoupled weight decay $0.1$ on matrices. Test cross-entropy is computed on held-out pairs; KL uses the exact transition matrix weighted by the empirical context distribution.

\end{document}